\documentclass{article}

\usepackage[preprint]{neurips_2025}

\usepackage{amsmath}
\usepackage{amssymb}
\usepackage{mathtools}
\usepackage{amsthm}
\usepackage{textcomp, gensymb}
\usepackage{algorithm2e}
\usepackage[normalem]{ulem}
\usepackage[version=3]{mhchem} %
\usepackage{siunitx} 

\usepackage{graphicx}
\usepackage{caption}
\usepackage{booktabs}

\usepackage{hyperref}
\usepackage[capitalize,noabbrev]{cleveref}

\usepackage{graphicx} %
\usepackage[utf8]{inputenc} %
\usepackage[T1]{fontenc}    %
\usepackage{hyperref}       %
\usepackage{url}            %
\usepackage{booktabs}       %
\usepackage{amsfonts}       %
\usepackage{nicefrac}       %
\usepackage{microtype}      %
\usepackage{xcolor}       
\usepackage{bm}             %
\usepackage{bbold}          %
\usepackage{thmtools, thm-restate}
\renewcommand{\P}{P}
\newcommand{\X}{\mathcal{X}}

\newcommand{\cov}[1]{C_{#1}}

\declaretheorem[name=Theorem,refname=Thm.]{theorem}
\declaretheorem[name=Lemma,sibling=theorem]{lemma}

\declaretheorem[name=Assumption,refname=Asm.]{assumption}

\crefname{assumption}{Asm.}{Asm.}
\crefname{equation}{}{}
\Crefname{equation}{Eq.}{Equations}
\crefname{figure}{Fig.}{Figs.}
\crefname{table}{Tab.}{Tabs.}
\crefname{section}{Sec.}{Sec.}
\crefname{theorem}{Thm.}{Thm.}
\crefname{lemma}{Lemma}{Lemmas}
\crefname{corollary}{Cor.}{Cor.}
\crefname{example}{Example}{Examples}
\crefname{remark}{Remark}{Remarks}
\crefname{algorithm}{Alg.}{Algorithms}
\crefname{appendix}{Appendix}{Appendices}
\crefname{subappendix}{Appendix}{Appendices}
\crefname{subsubappendix}{Appendix}{Appendices}
\newcommand{\msf}[1]{\mathsf{#1}}

\newcommand{\R}{{\mathbb{R}}}

\newcommand{\EE}{\mathbb{E}}
\newcommand{\PP}{\mathbb{P}}

\newcommand{\scal}[2]{\left\langle{#1},{#2}\right\rangle}

\newcommand{\norb}[1]{\bigl\|{#1}\bigr\|}

\newcommand{\tr}{\textrm{Tr}}

\newcommand{\hs}{\mathsf{HS}}

\newcommand{\eqals}[1]{\begin{align*}#1\end{align*}}
\newcommand{\eqal}[1]{\begin{align}#1\end{align}}

\newcommand{\Id}{\msf{Id}}

\renewcommand{\paragraph}[1]{{\bfseries #1.}}

\newcommand{\enc}{\varphi}
\newcommand{\evop}{{\msf{E}}}
\newcommand{\evopls}{E_{\enc}}
\newcommand{\Lsq}[1]{L^{2}(#1)}
\usepackage[textsize=tiny]{todonotes}

\title{Self-Supervised Evolution Operator Learning for High-Dimensional Dynamical Systems}

\author{
 Giacomo Turri$^{1}$ \\ \texttt{giacomo.turri@iit.it}  \And Luigi Bonati$^{2}$ \\ \texttt{luigi.bonati@iit.it}  \And Kai Zhu$^{3}$ \\ \texttt{22319143@zju.edu.cn} \\ \AND \textbf{Massimiliano Pontil}$^{1,4}$ \\ \texttt{massimiliano.pontil@iit.it} \\ \And \textbf{Pietro Novelli}$^{1}$ \\ \texttt{pietro.novelli@iit.it} \And \\
 $^1$CSML, Istituto Italiano di Tecnologia \quad $^2$ATSIM, Istituto Italiano di Tecnologia \\ $^3$CPS, Zhejiang University \quad $^4$AI Centre, University College London
}

\begin{document}

\maketitle

\begin{abstract}
We introduce an encoder-only approach to learn the evolution operators of large-scale non-linear dynamical systems, such as those describing complex natural phenomena. Evolution operators are particularly well-suited for analyzing systems that exhibit complex spatio-temporal patterns and have become a key analytical tool across various scientific communities. As terabyte-scale weather datasets and simulation tools capable of running millions of molecular dynamics steps per day are becoming commodities, our approach provides an effective tool to make sense of them from a data-driven perspective. The core of it lies in a remarkable connection between self-supervised representation learning methods and the recently established learning theory of evolution operators. To show the usefulness of the proposed method, we test it across multiple scientific domains: explaining the folding dynamics of small proteins, the binding process of drug-like molecules in host sites, and autonomously finding patterns in climate data. Code and data to reproduce the experiments are made available open source\footnote{\url{https://github.com/pietronvll/encoderops}}.
\end{abstract}

\section{Introduction}
Dynamical systems are fundamental to understanding phenomena across a vast range of scientific disciplines, from physics and biology to climate science and engineering. Traditionally, scientists have modeled these systems by formulating differential equations from first principles. However, as systems grow in scale and complexity, this approach quickly becomes computationally intractable and difficult to interpret, hindering the study of large-scale phenomena. Simultaneously, advancements in data collection techniques and computational power have led to an explosion of available data from experiments~\cite{Hersbach2020, ocp_dataset} and high-fidelity simulations~\cite{harvey2009acemd, Abraham2015, Eastman2017, bauer2015quiet}. This abundance of data makes data-driven approaches increasingly appealing for studying complex dynamics, with machine learning~\cite{ShalevShwartz2014} becoming a dominant paradigm for learning dynamical systems, largely focusing on predictive tasks such as forecasting. The recent revolution in data-driven weather modeling~\cite{pathak2022fourcastnet, bi2022pangu, lam2023learning, kochkov2024neural} stands as a paradigmatic example of ML's power in handling complex spatio-temporal dynamics. Similarly, reinforcement learning~\cite{sutton1998reinforcement} has reimagined control theory by leveraging data-driven strategies to optimize system behavior. While these data-driven methods excel at prediction and control, there remains a significant gap in approaches that offer interpretability. In scientific contexts, merely predicting system behavior is often insufficient; understanding why a system evolves in a certain way is paramount. For instance, comprehending the intricate dynamics of molecular interactions is crucial for understanding why a drug binds to a specific target or fails to do so, a level of insight not typically provided by black-box predictive models.

A modeling paradigm particularly well-suited for interpretability is that of {\em evolution operators}~\cite{Lasota1994, Applebaum2009, kovachki2023neural}. Under mild assumptions, dynamical systems and stochastic processes can be represented by a linear operator --- a mathematical entity that maps functions to other functions. This operator-based approach offers multiple advantages. First, it linearizes the dynamics, greatly simplifying tasks like forecasting and controller design. Second, these operators possess a spectral decomposition\footnote{A generalization of the eigenvalue decomposition of a matrix.}~\cite{Reed1972}, which expresses the system's complex dynamics as a linear combination of fundamental, coherent spatio-temporal modes~\cite{Molgedey1994}. Each mode represents a distinct, intrinsic pattern associated with a unique spatio-temporal structure defined in terms of growth or decay rates and oscillation frequencies. By identifying and analyzing these principal modes, researchers gain deep insights into the underlying mechanisms driving the system's macroscopic behavior, offering a structured, physically meaningful understanding.

Building on the understanding that evolution operators provide a powerful framework for interpretable analysis, significant effort has been directed towards learning these operators directly from data. Data-driven approaches for this task emerged already in the early 2000s, including pioneering work utilizing transfer operators for analyzing stochastic processes in computational biophysics \cite{schutte2001transfer}, as well as the dynamic mode decomposition family of methods \cite{schmid2010dynamic} for deterministic systems via the Koopman operator. In the ensuing years, there has been a significant acceleration in machine learning methods for evolution operator learning, encompassing theoretical advances through kernel methods and powerful end-to-end deep learning approaches.

In this work, we propose a method to learn evolution operators that builds upon these recent foundations while leveraging insights from the seemingly unrelated literature on self-supervised representation learning. Our primary contributions are twofold: 

\noindent\fbox{%
  \begin{minipage}{\dimexpr\linewidth-2\fboxsep-2\fboxrule\relax}
    \textbf{1.~}We introduce a paradigm for learning evolution operators that is based on self-supervised contrastive learning, and is amenable to scaling to large dynamical systems.\\\\
    \textbf{2.~}We expose a deep connection between evolution operator learning and contrastive self-supervised representation learning schemes.
  \end{minipage}%
}

To showcase the practical applicability of our approach, we study realistic dynamical systems drawn from atomistic simulations and climate modeling. We believe our work opens up multiple interesting follow-ups, both concerning methodological improvements in the training pipeline and conceptual advancements in harnessing the link between evolution operator learning and self-supervised representation.

\section{Evolution operators and how to learn them}\label{sec:the_operator_way}

Evolution operator learning is a data-driven approach to characterizing dynamical systems, either stochastic, $x_{t + 1} \sim p(\cdot | x_{t})$, or deterministic, $x_{t+1} \sim \delta(\cdot - F(x_{t}))$. Throughout, we assume the dynamics to be Markovian, so that the evolution of $x_t$ depends on $x_{t}$ alone and not on the states at times $s < t$. If this assumption is not satisfied by $x_{t}$, a standard trick is to re-define the state as a context $c_{t}^{H} = f(x_{t}, x_{t -1}, \ldots, x_{t - H})$ with history length $H$, where $f$ can be a simple concatenation, or a learned sequence model (e.g., a recurrent neural network or transformer).

{\em Evolution operators} are defined as follows: for every function $f$ of the state of the system, $(\evop f)(x_{t})$ is the expected value of $f$ one step ahead in the future, given that at time $t$ the system was found in $x_t$
\eqal{
(\evop f)(x_t) = \int p(dy | x_{t}) f(y) = \EE_{y \sim X_{t + 1} | X_{t}}[f(y) | x_t].
}
Notice that $\evop$ is an operator because it maps any function $f$ to another function, $x_{t} \mapsto (\evop f)(x_t)$, and is {\em linear} because $\evop(f + \alpha g) = \evop f + \alpha \evop g$. When the dynamics is deterministic, $\evop$ is known as the {\em Koopman operator}~\cite{Koopman1931}, while in the stochastic case it is known as the {\em transfer operator}~\cite{Applebaum2009}.  

Evolution operators fully characterize the dynamical system because knowing $\evop$ allows us to reconstruct the dynamical law $p(\cdot | x_{t})$. Indeed, for any subset of the state space $B \subseteq \X$, applying $\evop$ to the indicator function of $B$, we have
\eqals{
    (\evop 1_{B})(x_t) = \int_{B}p(dy|x_{t}) = \PP\left[X_{t + 1} \in B | x_t\right].
}
An advantage of the operator approach over dealing directly with the conditional probability $p(\cdot | x_{t}$) is that $\evop$ acts linearly on the objects to which it is applied. This means that operators unlock an arsenal of tools from linear algebra and functional analysis, which would be unavailable otherwise. Arguably the most important of them is the spectral decomposition, allowing us to decompose $\evop$, and hence the dynamics, into a linear superposition of dynamical modes. These ideas lie at the core of the celebrated Time-lagged Independent Component Analysis~\cite{Molgedey1994,perez2013identification,bonati2021deep}, and  Dynamical Mode Decomposition~\cite{Schmid2010, Kutz2016}. 
\subsection{Learning $\evop$ and its spectral decomposition from data}
We now review the main approaches to learn the evolution operator and its spectral decomposition from a finite dataset of observations, with an emphasis on the least squares approach, which is essential to understand every other method as well.
A core idea of operator learning is that operators are defined by how they act {\em on a suitable linear space of functions}, similarly to how matrices are defined by their action on a basis of vectors. Of course, not every function $f$ is interesting, and this nicely parallels with the matrix example, where the most "interesting" directions are those that recover most of the variance in the data. Learning $\evop$, therefore, is usually cast as the following problem:

\noindent\fbox{%
  \begin{minipage}{\dimexpr\linewidth-2\fboxsep-2\fboxrule\relax}
  Letting $\enc(x) \in \R^{d}$ be a --- learned or fixed --- encoder of the state, find the best approximation of $\evop$ {\em restricted} to the $d$-dimensional linear space of functions generated by $\enc$, given the data.
  \end{minipage}%
}

In practice, the data is usually a collection of transitions $\mathcal{D} = (x_i, y_i)_{i = 1}^{N}$, where it is intended that $x_{i} \sim \PP[X_{t}]$ are sampled from a distribution of initial states, while $y_{i} \sim p( \cdot | x_{i})$. 

\paragraph{Least squares} In this approach the encoder $\enc$ is a frozen, that is non-learnable, dictionary of functions, and we are interested in approximating the action of $\evop$ on functions of the form $f(x) = \scal{w}{\enc(x)}$ for every $w \in \R^{d}$. To this end, one minimizes the empirical error between the true conditional expectation $\EE_{y \sim X_{t + 1} | X_{t}}[\scal{w}{\enc(y)} | x]$, and a linear model $ \scal{Ew}{\enc(x)}$, where the matrix $E \in \R^{d \times d}$ identifies the restriction of the evolution operator to the linear span of the dictionary:
\eqal{
\label{eq:least-squares-loss}
    \frac{1}{N}\sum_{i = 1}^{N} (\scal{w}{\enc(y_{i})} - \scal{Ew}{\enc(x_i)})^2 &= \frac{1}{N}\sum_{i = 1}^{N}\scal{w}{\enc(y_i) - E^\top \enc(x_i)}^{2} \nonumber \\
    & \leq \frac{1}{N}\sum_{i = 1}^{N}  \norb{\enc(x_i) - E^\top \enc(y_i)}^{2} + \lambda \norb{\enc(x_i)}^{2}.
}
In the second line, we assumed $\norb{w} \leq 1$, used the Cauchy–Schwarz inequality, and added a ridge penalty. The minimizer of~\cref{eq:least-squares-loss} can be computed in closed form~\cite[][and references therein]{korda2018convergence, kostic2022learning} as%
\eqal{
\label{eq:least-squares-estimator}
\evopls = (\cov{X} + \lambda\Id)^{-1}\cov{XY}, \quad \text{with}~~\cov{XY} = \frac{1}{N} \sum_{i = 1}^N\enc(x_i)\enc(y_i)^\top ~\text{and} ~~\cov{X} = \cov{XX}.
}
In the limit of infinite data, $N\to \infty$, and infinitely dimensional encoders, $d \to \infty$, the least squares estimator converges~\cite{korda2018convergence} in operator norm to the evolution operator $\evop$, and similar asymptotic convergence results are proved for its spectrum. 

\paragraph{Mode decomposition} 
The spectral decomposition of $\evop$ is approximated by expressing the least-squares estimator in its eigenvectors' basis $\evopls = Q\Lambda Q^{-1}$, where the columns of $Q = [q_1, \cdots, q_d]$ are the eigenvectors of $\evopls$, and $\Lambda$ is a diagonal matrix of eigenvalues. In this basis, the expected value in the future for a function $f(x) = \scal{w}{\enc(x)}$ is expressed as
\eqal{
    \label{eq:mode-decomp}
    \EE_{y \sim X_{t + 1} | X_{t}}[f(y) | x] \approx \scal{\evopls w}{\enc(x)} = \scal{Q\Lambda Q^{-1}w}{\enc(x)}
    = \sum_{i = 1}^{d} \lambda_i \scal{q_i}{\enc(x)} (Q^{-1}w)_{i}.
}

The spectral decomposition expresses the transition $x_t \to x_{t+1}$ as a sum of \emph{modes} of the form $\lambda_i \scal{q_i}{\enc(x)} (Q^{-1}w)_i$, each of which can be broken down into three components:
\begin{enumerate}
\item The eigenvalues $\lambda_i$ determine the time scales of the transition. Indeed, applying the evolution operator $s$ times to analyze the transition $x_t \to x_{t+s}$ leaves~\cref{eq:mode-decomp} unchanged, except that each $\lambda_i$ becomes $\lambda_i^s$. Writing $\lambda_i^s = \rho_i^s e^{i s \omega_i}$ in polar coordinates, reveals that the modes decay exponentially over time with rate $\rho_i$, while oscillating at frequency $\omega_i$.

\item The initial state $x$ influences the decomposition through the factor $\Psi_{i}(x) = \scal{q_i}{\enc(x)}$. This coefficient captures how strongly the state $x$ aligns with the $i$-th mode. When $q_i$ corresponds to an eigenvalue with slow decay, i.e., $|\lambda_i| \approx 1$, the term $\Psi_{i}(x)$ serves as a natural quantity for clustering states into \emph{coherent} or \emph{metastable} sets.

\item The coefficient $(Q^{-1}w)_i$, in turn, indicates how the function represented by the vector $w$ relates to the $i$-th mode. This connection makes it possible to link the dynamical patterns to specific functions --- or {\em observables} -- thereby deepening our understanding of the system.
\end{enumerate}

\paragraph{Kernel methods} Leveraging the kernel trick, one can learn evolution operators by deriving a closed-form solution of~\cref{eq:least-squares-loss} in terms of kernel matrices whose elements are of the form $k(x_i, x_j) = \scal{\enc(x_i)}{\enc(x_j)}$, with $k(\cdot, \cdot)$ a suitable kernel function. Thanks to the theory of reproducing kernel Hilbert spaces, this class of methods is backed up by statistical learning guarantees, such as the ones derived in~\cite{kostic2022learning, kostic2023sharp,nuske2023finite}. Similarly to the least-squares approach, one also approximate the spectral decomposition of $\evop$ via kernel methods, and this task captured quite a lot of attention of the researchers in this area, see~\cite{williams2014kernel,Kawahara2016, Klus2019,Das2020, Alexander2020, meanti2023estimating}.

\paragraph{Deep learning} In contrast to the above discussion, where the encoder $\enc$ is prescribed, a number of methods proposed to approximate $\evop$ from data with end-to-end schemes including $\enc$ as a learnable neural network. Since learning $\evop$ ultimately entails learning its action on the linear space spanned by $\enc$, it is appealing to choose an encoder capturing the most salient features of the dynamics. To this end, one can train $\enc$ via an encoder-decoder scheme as proposed in~\cite{lusch2018deep, azencot2020forecasting, wehmeyer2018time, frion2024neural} or with encoder-only approaches as in~\cite{mardt2018vampnets,kostic2023dpnets,federici2023latent}.

In encoder-decoder schemes, $\enc$ is trained alongside a decoder network, minimizing a combination of prediction and reconstruction errors. Notice that trying to minimize the prediction error~\cref{eq:least-squares-loss} alone immediately leads to a {\em representation collapse} with $\enc$ mapping every input to $0$ (and getting a prediction error of 0). Having to train a decoder network results in a model twice as big\footnote{Assuming that the parameter count between encoder and decoder is roughly equivalent.} compared to the encoder-only methods, making it less suitable for very large-scale applications. More fundamentally, though, minimizing a reconstruction loss is desirable when $\evop$ is employed for forecasting tasks, but might be useless or even detrimental~\cite{lyu2023taskoriented,schwarzer2020data} whenever one is interested in tasks such as interpretation and control of dynamical systems. 

Arguably, the main advantages of evolution operators over other techniques such as ~\cite{pathak2022fourcastnet, lam2023learning, pfaff2020learning, sanchez2020learning, li2020fourier} come from the associated spectral decomposition, and {\em not} from a superior performance in forecasting. Encoder-only approaches follow this intuition and prioritize approximating the spectral decomposition of $\evop$ over the raw forecasting performances. Concretely, this is accomplished via loss functions that are minimized when $\enc$ spans the leading singular space of  $\evop$. Clearly, once an encoder has been trained, one can freeze it and use it in conjunction with the techniques discussed for the least-squares approach.

In this work, we propose an encoder-only method based on a loss function originally designed for self-supervised representation learning. Though our approach is broadly applicable, we mainly focus on applications involving interpretability and model reduction of scientific dynamical systems, highlighting how ML evolution operators can help in advancing fundamental science. Promising directions in reinforcement learning~\cite{lyu2023taskoriented,schwarzer2020data,rozwood2024koopman,novelli2024operator} and control, however, could also benefit from this approach, and we leave these for future work.
\section{Contrastive learning for deep evolution operators}  \label{sec:learn}

\begin{algorithm}
\noindent %
\begin{minipage}[t]{0.55\textwidth}
\DontPrintSemicolon
\SetAlgoLined

\For{$k=1$ \KwTo \texttt{num\_steps} }{ 
$\mathcal{B} \gets \{(x_i, y_i) \sim\mathcal{D}\}_{i=1}^B$
\ForAll{$i$}{
$z_i \leftarrow \enc(x_i)$ and $q_{i} \leftarrow \P\enc(y_i)$ 
}
$r_{ij} \leftarrow \scal{z_{i}}{q_{j}}$ \\
$d\enc, d\P \leftarrow  \nabla \left[ \frac{1}{B(B-1)}\sum_{i\neq j} r_{ij}^2 - \frac{2}{B}\sum_{i} r_{ii}\right]$ \\
$\enc, \P \leftarrow \mathrm{opt}(\enc, \P,  d\enc, d\P)$
}
\end{minipage}\, \, 
\begin{minipage}{0.38\textwidth}
    \hspace{-1cm}
    \includegraphics[width=1.3\textwidth]{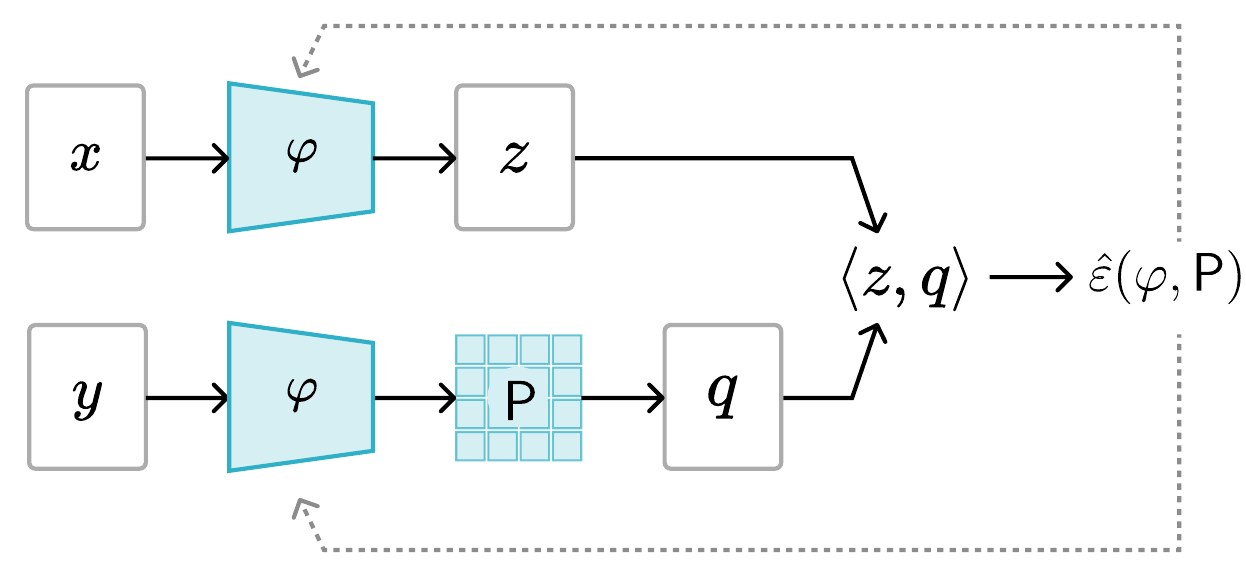}
    \vspace{-2.8cm}
\end{minipage}
\caption{\label{alg:main} A pair of consecutive observations $(x, y)$ from a dynamical system are mapped to representations $z$ and $q$ via an embedding function $\varphi$. The representation $q$ is also processed by a predictor $\P$. The algorithm iteratively optimizes $\enc$ and $\P$ using the contrastive objective~\cref{eq:loss} based on the similarity $\langle z, q \rangle$.}
\end{algorithm}

As discussed above, we are interested in the evolution operator
\eqal{
\label{eq:evop_ratioform}
(\evop f)(x_t) = \EE_{y \sim X_{t + 1} | X_t} \left[ f(y) | x_{t} \right] = \EE_{y \sim X_{t + 1}}     \left[\frac{p(y | x_t)}{p(y)} f(y)\right],
}
where in the last equality we expressed the expectation in the form of an importance sampling estimator with respect to the probability of the future state $\PP(X_{t + 1})$. In so doing, we link the evolution operator $\evop$ to the density ratio 
\begin{equation}
    \label{eq:density_ratio}
r(x_t, x_{t + 1}) = \frac{p(x_{t + 1} | x_{t})}{p(x_{t + 1})}.
\end{equation}
In our self-supervised scheme, see Alg.~\ref{alg:main}, we optimize a model for the density ratio~\cref{eq:density_ratio} parametrized as the bilinear form $\scal{\enc(x_{t})}{\P\enc(x_{t + 1})}$. Here, $\enc$ is a $d$-dimensional encoder, while $\P$ is a linear {\em predictor} layer which, as discussed below, approximates the action of $\evop$ on the linear subspace of functions spanned by $\enc$.

One motivation behind our approach is provided by Contrastive Predictive Coding~\cite{oord2018representation}, where the density ratio $r(x_{t}, x_{t + 1})$ is noted to be a proxy for the mutual information between $x_{t}$ and $x_{t +1}$: 
\eqals{
I(X_{t}, X_{t + 1}) = \EE_{(x, y) \sim (X_t, X_{t + 1})}\left[\log \left(r(x, y)\right)\right].
}
Having a good approximation of the density ratio is a necessary condition to ensure that the features from our encoder $\enc$ retain the information on the evolution $x_{t} \to x_{t +1}$. As shown by~\citep{tschannen2020mutual}, however, retaining information between the covariates is not sufficient to have a strong representation on downstream tasks. To address this concern, in place of the commonly used contrastive losses -- which are lower bounds of the mutual information~\cite{oord2018representation, gutmann2010noise, poole2019variational} -- we directly optimize the $L^2$ error between the density ratio and our bilinear model $\scal{\enc(x_{t})}{\P\enc(x_{t + 1})}$:
\eqal{
\label{eq:abstract_loss}
\varepsilon(\enc, \P) = &\EE_{(x, y) \sim X_{t} \otimes X_{t + 1}}\left[ \left(r(x, y) - \scal{\enc(x)}{\P\enc(y)}\right)^{2}\right] \nonumber \\
= & \EE_{(x, y) \sim X_{t} \otimes X_{t + 1}}\left[ \scal{\enc(x)}{\P\enc(y)}^2\right]  -2\EE_{(x, y) \sim (X_t, X_{t + 1})}\left[ \scal{\enc(x)}{\P\enc(y)}\right] + \text{cst.},
}
where $\EE_{X_{t} \otimes X_{t + 1}}$ is the expected value between the product of the marginals $X_{t}$ and $X_{t + 1}$\footnote{That is, the product measure $\PP[X_{t}] \otimes \PP[X_{t +1}]$}, while the second line is obtained by expanding the square and using Bayes theorem. The choice of~\cref{eq:abstract_loss} in place of more common contrastive losses is motivated by the following theoretical observation~\cite{wang2022spectral, kostic2024neural}: if the evolution operator $\evop$ is Hilbert-Schmidt, the loss \cref{eq:abstract_loss} can be equated\footnote{A formal and self-contained proof of this statement is reported in the Supplementary Material.} to a proper operator-learning error $\norb{\evop - \Phi^* \P \Phi}_{{\rm HS}}^2$, where $\Phi f \in \R^{d}$ with components $\scal{f}{\enc_{i}}_{L^{2}(X_t)}$. The loss~\cref{eq:abstract_loss}, therefore, is a natural choice given the context of this work.

Estimating the squared term in~\cref{eq:abstract_loss} via U-statistics~\cite{hoeffding1992class} and foregoing the constant term, we finally get to the empirical loss  
\eqal{\label{eq:loss}
\hat{\varepsilon}(\enc, \P) = \frac{1}{N(N-1)}\sum_{i\neq j} \scal{\enc(x_i)}{\P\enc(y_j)}^2 - \frac{2}{N}\sum_{i = 1}^{N} \scal{\enc(x_i)}{\P\enc(y_i)}.
}

Let us notice that~\cref{eq:loss}, has been proposed independently multiple times across different corners of machine learning research. For example, it drives the spectral contrastive learning approach of~\cite{haochen2021provable, haochen2022beyond}, it has been used in the context of reinforcement learning in~\cite{ren2022spectral}, causal estimation in~\cite{sun2024spectral}, and more generally in conditional expectation models in~\cite{wang2022spectral, kostic2024neural}. Recently,~\citep{lu2024f} showed that~\cref{eq:loss} belongs to a wide class of contrastive learning losses defined by Csisz\'ar $f$-divergences.

Plugging our model back into~\cref{eq:evop_ratioform}, the evolution operator gets parametrized as
\eqals{
(\evop f)(x_t) = \EE_{y \sim X_{t + 1}} \left[ r(x_t, y)f(y)\right] \approx \EE_{y \sim X_{t + 1}} \left[ \scal{\enc(x_t)}{\P\enc(y) }f(y)\right].
}
Further, if the function $f$ is in the linear span generated by the encoder $f(y) =\scal{\enc(y)}{w}$, we can simplify the expression above as,
\eqals{
(\evop f)(x_t) = \scal{\enc(x_t)}{\P \left(\EE_{y \sim X_{t + 1}} \left[\enc(y)\enc(y)^\top\right]\right) w} = \scal{\enc(x)}{ \P \cov{Y} w},
}
where we introduced the covariance of the futures $\cov{Y} = \EE_{y \sim X_{t + 1}}[\enc(y)\enc(y)^\top]$. Thus, the matrix $\evopls = \P \cov{Y}$ provides an approximation of the evolution operator $\evop$ in the finite-dimensional space generated by the state representation $\enc$. 

Remarkably, for any fixed $\enc$, the $\P$ minimizing \cref{eq:abstract_loss} can be computed in closed form (see Supplementary Material), and is given by $\P_{*} = \cov{X}^{-1} \cov{XY} \cov{Y}^{-1}$, so that the model for the evolution operator is given by
\eqal{
\label{eq:ls-from-predictor}
\evopls = \P_{*}\cov{Y} = \cov{X}^{-1} \cov{XY} = \text{Eq. \cref{eq:least-squares-estimator} with } \lambda \to 0,
}
coinciding with the least-squares estimator discussed above.

\paragraph{Connection with the VAMP-2 Score} Originally introduced as a representation learning scheme for molecular kinetics, the VAMP-$r$ scores~\cite{wu2020variational} are principled metrics to evaluate the quality of an encoder $\enc$ in the context of evolution operator learning. In particular, the VAMP-2 score can be defined in terms of covariances as
\eqal{
\label{eq:VAMP_score}
 \text{VAMP}_{2}(\enc) = \|\cov{X}^{-1/2}\cov{XY}\cov{Y}^{-1/2}\|_{\hs}^{2}.
}
By noticing that the loss function~\cref{eq:abstract_loss} can be equivalently rewritten as
\eqals{
    \varepsilon(\enc, \P) = \tr[\P^{\top}\cov{X}\P\cov{Y} - 2\P\cov{YX}],
}
and substituting the optimal predictor $\P_{*} = \cov{X}^{-1} \cov{XY}$ inside this expression, we obtain the remarkable identity
\eqals{
\varepsilon(\enc, \P_{*}) = -\|\cov{X}^{-1/2}\cov{XY}\cov{Y}^{-1/2}\|_{\hs}^{2} = -\text{VAMP}_{2}(\enc).
}
Our loss function, therefore, matches the negative VAMP-2 score when $\P$ is optimal. Compared to methods that directly maximize the VAMP score such as~\cite{mardt2018vampnets}, however, our approach do not require matrix inversions in the computation of the loss~\cite{wu2020variational}, an operation which is unwieldy and prone to instabilities\footnote{Backpropagation through inversions may lead to gradient explosion~\cite{Golub1973}.} in large-scale applications. Instead, the loss function~\eqref{eq:loss} is written in terms of simple matrix multiplications, making it perfect for GPU-based training.

\subsection{Practical implementation}
The implementation of our method, summarized in Alg.~\ref{alg:main}, follows standard self-supervised learning procedures~\cite{chen2020simple, grill2020bootstrap,zbontar2021barlow,chen2021exploring}. There, a {\em positive pair} of data-points, in our case a pair of consecutive observations of the dynamical system, are processed through an encoder network $\enc$, and optionally a predictor network, which in our case is a simple linear layer $P$. We apply simplicial normalization~\cite{lavoie2022simplicial} to the outputs of the embedding $\enc$.
To keep our implementation as close to the theoretical insights as possible, we didn't concatenate additional projection heads to the encoder $\enc$, as suggested in~\cite{chen2020simple, grill2020bootstrap, zbontar2021barlow,chen2021exploring}. Furthermore, because of the identity~\cref{eq:ls-from-predictor} we kept $\P$ linear, despite tiny MLPs being usually employed for the predictors. While experimenting with these architectural variants is out of the scope of this work, we believe it is an exciting direction for future work.

To get the least-squares estimator $\evopls$ associated with $\enc$, different options are available. One can make use of the closed form expression~\cref{eq:least-squares-estimator} by computing the covariances at the end of the training of $\enc$. This approach, however, requires a full forward pass over the full training dataset, which might be impractical for very large datasets. Another option is to use~\cref{eq:ls-from-predictor}, but this again requires the evaluation of the covariance, and might be suboptimal whenever $\P$ isn't yet converged to the true minimizer $\P_{*}$. In our implementation, instead, we kept two buffers for $\cov{X}$ and $\cov{XY}$, which are updated online during the training loop via an exponentially moving average of the batch covariances. At the end of the training, we use buffers to compute $\evopls$ as in~\cref{eq:least-squares-estimator}.

\section{Experiments} \label{sec:exp}
\subsection{Warm up: Lorenz '63}
As an appetizer, we evaluated our method on the Lorenz ’63 system \cite{Lorenz1963-LORDNF}, a classical example of a chaotic dynamical system governed by three coupled ordinary differential equations. %
To validate the performance of our approach, we tested it on a one-step-ahead forecasting task, and we analyzed the learned dynamical modes. Because of the low-dimensionality of the state $x_{t}$, we appended it as a non-learnable feature of the encoder $\enc(x_t) = [{\rm MLP}(x_{t}), x_{t}]$ to ensure that the forecasting target--the state itself--lies in the linear space of functions spanned by $\enc$ by design. The learnable part of the encoder consisted of a small multi-layer perceptron (MLP).%

In \cref{tab:L63}, we compare the performance of the estimator $\evopls$ from~\cref{eq:least-squares-estimator}, with an encoder $\enc$ trained according to Alg.~\ref{alg:main}, against several baseline models. These include Linear Least Squares (LinLS), Kernel Ridge Regression~\cite{kostic2022learning} (KRR) with a Gaussian kernel, VAMPNets~\cite{mardt2018vampnets}, Dynamic Autoencoder~\cite{lusch2018deep} (DAE), and Consistent Autoencoder~\cite{azencot2020forecasting} (CAE). To ensure a fair comparison, we matched the encoder architecture for VAMPNets, DAE and CAE, while decoders of DAE and CAE were defined as MLPs symmetric to their respective encoders. For KRR, the rank was set equal to the latent dimensionality used in the deep learning models. %

The results on the forecasting task demonstrate that, although our model is not specifically designed for prediction, it achieves the best performance among all considered methods. Furthermore, when comparing the runtimes of the deep learning-based approaches, our implementation of Alg.~\ref{alg:main} is faster than VAMPNets, DAE, and CAE. Finally, we verified that the leading eigenfunctions obtained by our approach correctly identify coherent sets on the stable attractor (see~\cref{fig:L63}).

\begin{figure}
  \begin{minipage}[b]{.75\linewidth}
\centering
\footnotesize
\setlength{\tabcolsep}{2pt}
\begin{tabular}{lcccccc}
\toprule
 & \textbf{Ours} & LinLS & KRR & VAMPNets & DAE & CAE \\
\midrule
RMSE &
\textbf{0.48{\scriptsize$\pm$0.09}} &         
1.14 & 
2.31 & 
0.65{\scriptsize$\pm$0.11} &
0.79{\scriptsize$\pm$0.51} & 
2.14{\scriptsize$\pm$0.40}
\\ 
Runtime (s) & 19.4{\scriptsize$\pm$0.2} & 
\textbf{(4.7{\scriptsize$\pm$1})$\mathbf{\times10^{-4}}$} & 
25.9{\scriptsize$\pm$0.9} & 
27.7{\scriptsize$\pm$0.1} &
29.9{\scriptsize$\pm$0.3} &
35.0{\scriptsize$\pm$0.1}
\\
\bottomrule
\end{tabular}\captionof{table}{Forecasting errors and training times for the Lorenz '63 example (averaged over 20 independent runs). RMSE values are scaled by $10^{-2}$.}\label{tab:L63}    
\end{minipage}\hfill
  \begin{minipage}[b]{.22\linewidth}
    \centering
    \includegraphics[width=\linewidth, height=1.8cm]{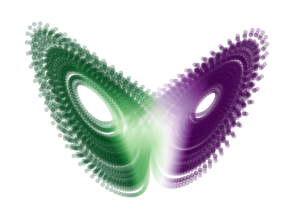}
    \captionof{figure}{2\textsuperscript{nd} leading eigenfunction.} \label{fig:L63} 
  \end{minipage}
\end{figure}

\subsection{High-resolution dynamical modeling of protein folding}

\begin{figure}[ht!]
\centering
\includegraphics[width=\textwidth]{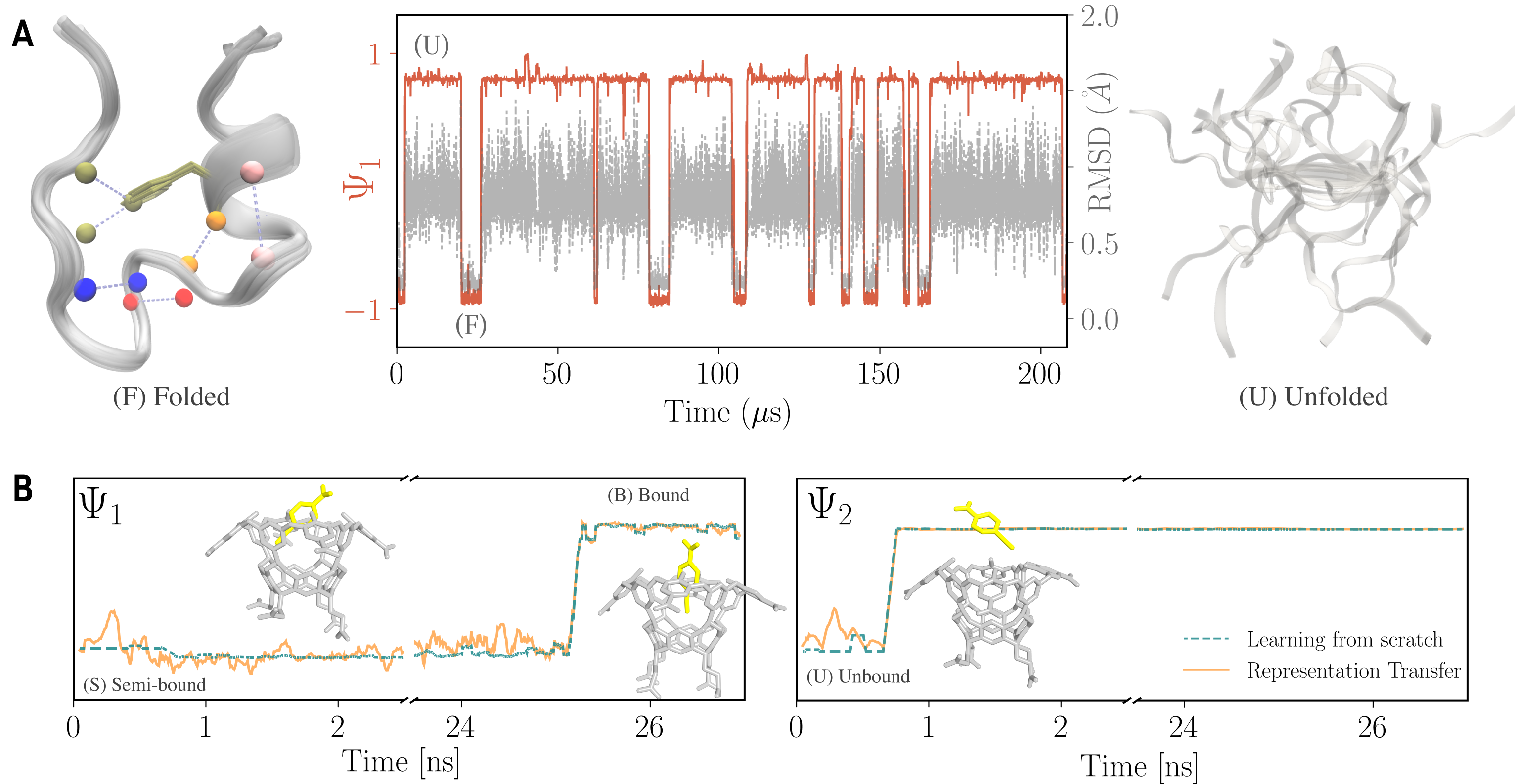}
\caption{Slow dynamical modes of biomolecular processes.  
\textbf{A}: \textit{Trp-Cage folding}. Time series of the leading eigenfunction $\Psi_1$ (red, left axis) alongside RMSD (gray, right axis), capturing transitions between folded (F) and unfolded (U) states. Representative snapshots of each state are shown. In the folded structure, key hydrogen bonds identified as relevant by the LASSO model are highlighted.  
\textbf{B}: \textit{Calixarene binding}. Eigenfunctions $\Psi_1$ (left) and $\Psi_2$ (right) capture ligand transitions from unbound (U) to semi-bound (S) and bound (B) states. The model using a representation transferred from other ligands (solid line) closely matches one trained from scratch (dashed). Right: representative structures corresponding to each metastable state.
}
\label{fig:results}
\end{figure}

The Trp-Cage miniprotein is a widely studied benchmark for protein folding due to its small size and fast dynamics~\cite{lindorff2011fast}. Previous works, including SRV-based Markov State Models~\cite{sidky2019high} and GraphVAMPNet~\cite{ghorbani2022graphvampnet}, have modeled Trp-Cage dynamics using coarse-grained representations, where the state of the system is defined by the small subset of 20 C$_{\alpha}$ atoms in the backbone of the protein. 
Our approach allows us to scale to a more expressive molecular representation based on all 144 heavy atoms, employing the SchNet \cite{schutt2017schnet} graph neural network architecture as the encoder $\enc$. 
After training, we calculate the eigenvalue decomposition of the evolution operator as described in~\cref{sec:the_operator_way}. As shown in Fig.~\ref{fig:results}A, the leading eigenfunction $\Psi_{1}(x) = \scal{q_{1}}{\enc(x)}$ correlates strongly with the system’s root-mean-square deviation (RMSD) from the folded structure, confirming that $\Psi_1$ encodes the folding-to-unfolding transition. Clustering the molecular configurations according to the values of $\Psi_1$ reveals a clear separation between folded and unfolded ensembles (see snapshots in Fig.~\ref{fig:results}A).

To interpret the nature of this slow mode, we regress $\Psi_1$ against a library of physically meaningful descriptors—specifically, hydrogen bond interactions across residue pairs—using a sparse LASSO model~\cite{brunton2016discovering, zhang2024descriptor, novelli2022characterizing}. This analysis reveals a network of hydrogen bonds stabilizing the folded state, including contributions from side-chain interactions that would be invisible to coarse-grained dynamical models such as~\cite{ghorbani2022graphvampnet}. Finally, we note that the implied timescale\footnote{The implied timescale can be computed from the eigenvalues as $\tau_{i} = -\Delta t / \log (\lambda_{i})$, where $\Delta t$ is the time lag between two consecutive observations.} $\tau_1$ derived from the leading eigenvalue of the learned operator is approximately 2.5~$\mu s$. This is higher than the 2~$\mu s$ timescale obtained by GraphVAMPNet using the same GNN architecture but trained only on C$_\alpha$ inputs~\cite{ghorbani2022graphvampnet}. According to the variational principle for Markov processes~\cite{wu2020variational, noe2013variational}, higher implied timescales indicate a better approximation of the system’s true slow dynamics. This underscores the advantage of a fine-grained representation in capturing accurate and interpretable slow dynamics.
\vspace{-0.5cm}
\subsection{Learning transferable representations for the binding of small molecules}
Our second case study focuses on the binding of small molecules to a calixarene-based system~\cite{yin2017overview}, which is often used as a simplified model to study the dynamical processes relevant, for instance, in drug design. 
Our baseline is obtained by using Alg.~\ref{alg:main} to train an encoder $\enc$ on molecular dynamics data describing the binding dynamics of a single molecule (G2) to the host system. As in the previous example, we employ a SchNet architecture for $\enc$.
As shown in Fig.~\ref{fig:results}B, the slowest dynamical mode, captured by the dominant eigenfunction $\Psi_1$, is associated with a transition between a semi-bound configuration and the fully bound state. Structural inspection reveals that this intermediate state corresponds to a misaligned pose of the guest, caused by the presence of a water molecule occupying the binding pocket. The second eigenfunction $\Psi_2$ instead resolves the unbound-to-bound transition. Our findings align with previous works~\cite{rizzi2021role}, where water occupancy was identified as a key kinetic bottleneck in host–guest interactions.

We now turn to a key question: can a representation trained on one set of molecular systems generalize to others? This capability is essential for scalable modeling in applications like drug discovery, where retraining a model for every new compound is prohibitive. To test the transferability of the representations $\enc$ trained with our method, we conduct the following experiment: we train the encoder on molecular dynamics simulations for two molecules (G1 and G3), then freeze it, and use it to analyze the binding dynamics of a different ligand (G2). Using the frozen encoder, we compute the evolution operator of (G2) via~\cref{eq:least-squares-estimator}, and examine its dominant eigenfunctions. Remarkably, the transferred representation successfully recovers the key dynamical modes of the binding process of (G2) without having seen it during the representation learning phase. In particular, it recovers both the entry of the guest molecule into the host cavity and its final locked configuration (Fig~\ref{fig:results}B). Although this is a simplified setting, the result illustrates that our self-supervised model learns features that are not only informative but also transferable across related molecular systems.
\subsection{Patterns in Global Climate}
\begin{figure}[ht!]
\centering
\includegraphics[width=\textwidth]{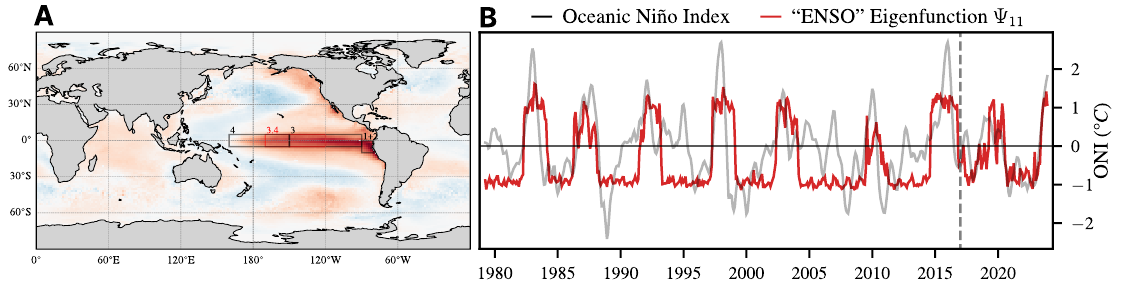}
    \caption{ENSO mode from an encoder trained with Alg.~\ref{alg:main}. 
    \textbf{A} Mode associated with the 11\textsuperscript{th} eigenfunction, highlighting dominant activation in the tropical Pacific. Boxes indicate standard ENSO monitoring zones.
    \textbf{B} Right eigenfunctions corresponding to the 11\textsuperscript{th} eigenvalues, compared to the ONI index (black). The vertical line marks the split between training and validation sets.}
    \label{fig:ENSO}
\vspace{-0.5cm}
\end{figure}
The El Ni\~no–Southern Oscillation (ENSO) is one of the most influential sources of interannual climate variability~\cite{diaz2000nino, glantz2020reviewing, callahan2023persistent}, arising from coupled ocean–atmosphere dynamics in the tropical Pacific \cite{bjerknes1969atmospheric, philander1983nino}. %
Characterizing ENSO is a central goal in climate science, particularly in the context of its potential changes under global warming \cite{mcphaden2006enso, cai2021changing}. Traditional approaches often relied on hand-crafted indices to isolate ENSO-related variability \cite{wang2017nino, timmermann2018nino}, but recent efforts have turned to operator-theoretic and machine learning methods for a more systematic understanding \cite{froyland2021spectral, ham2019deep}. %
ENSO is conventionally characterized using monthly-averaged SST anomalies, denoted as SST*. Following the procedure outlined in \cite{noaaClimatePrediction}, we compute SST* from SST data obtained from the ORAS5 reanalysis \cite{zuo2019ecmwf}, as made available in the ChaosBench dataset \cite{nathaniel2024chaosbench}.%

For this experiment, we trained a ResNet-based~\cite{he2016deep} encoder on the SST* dataset via Alg.~\ref{alg:main}. The model was trained using a lag time of 1 month on data from 1979 to 2016, while the 2017–2023 period is reserved for validation. %
After learning the representation, we estimate the transfer operator as described in \cref{eq:least-squares-estimator} and compute its spectral decomposition to extract dominant modes~\cref{eq:mode-decomp}, see \cref{fig:ENSO}. As expected, we recover modes corresponding to known periodicities in the climate system, including annual, seasonal, and sub-seasonal oscillations (see Supplementary Material). Remarkably, one of the leading nontrivial modes (11\textsuperscript{th} in magnitude) clearly reflects ENSO dynamics.  Its associated right eigenfunction exhibits a strong correlation with the Oceanic Ni\~no Index (ONI), a widely used metric for ENSO monitoring, %
while the corresponding spatial mode shows dominant activation over the tropical Pacific region (\cref{fig:ENSO}A). This result underscores the model’s ability to autonomously identify in an unsupervised manner complex climate phenomena without prior localization~\cite{froyland2021spectral, lapo2025method}. Furthermore, our method generalizes effectively to unseen data, successfully detecting the 2023 El Ni\~no event within the validation period.
\vspace{-0.3cm}
\section{Conclusion}
\label{sec:conclusions}
In this work, we introduced a framework for learning evolution operators through self-supervised contrastive learning, bridging two previously disconnected research areas: data-driven dynamical systems modeling and representation learning. Our method scales effectively to large and complex systems while preserving the interpretability afforded by evolution operators. By uncovering a connection between contrastive learning objectives and the spectral properties of evolution operators, we enable new tools for analyzing high-dimensional, spatio-temporal data in a physically meaningful way. Our experiments on atomistic and climate systems demonstrate the versatility and promise of this approach. Looking ahead, this connection opens the door to more expressive learning architectures, robust training strategies, and broader applications in scientific discovery and control. 

\paragraph{Limitations} We highlight key limitations of our method. The identification of~\cref{eq:loss} with a proper evolution operator loss holds only if $\evop$ is Hilbert-Schmidt—an assumption often violated by deterministic dynamical systems. Methodologically, we opted for simplicity and did not incorporate several architectural and optimization advances from self-supervised learning that could improve training. As detailed in the Supplementary Material, the climate example was particularly challenging and would benefit from a more robust pipeline. Finally, due to the nature of our experiments, evaluation was more qualitative than typical in ML; designing simple benchmarks to assess the learned spectral decomposition would benefit the broader evolution operator learning community.

\paragraph{Acknowledgements} We acknowledge the CINECA award under the ISCRA initiative (grant \texttt{IscrC\_LR4LSDS}). This work was also partially funded by the European Union - NextGenerationEU initiative and the Italian National Recovery and Resilience Plan (PNRR) from the Ministry of University and Research (MUR), under Project PE0000013 CUP J53C22003010006 "Future Artificial Intelligence Research (FAIR)", and European Project ELIAS N. 101120237. 

\bibliographystyle{unsrt}
\bibliography{bibliography}

\newpage
\appendix
\onecolumn
\section*{\Huge Supplementary Material}
\section{Theory: connection between operator learning and self-supervised learning}
The material presented in this section is based on the results of~\cite[Lemma 3.2]{haochen2021provable}, ~\cite[Lemma 4.1]{wang2022spectral}, and~\cite[Theorem 1]{kostic2024neural}. We will show that the loss~\cref{eq:abstract_loss} is equivalent to an operator learning loss, in which we try to regress the $\evop$ directly.

Define $\nu = \PP[X_{t}]$, the distribution of the initial states in our dataset, and $\mu = \EE_{x \sim \nu}[p(\cdot | x)] = \PP[X_{t +1}]$ the distribution of the evolved states. In practice, $\nu$ can be the following:
\begin{itemize}
    \item If a simulator is available, $\nu$ can be {\em any} distribution of initial states, and $\mu$ is obtained by a single step of the simulator on data from $\nu$.
    \item If one samples trajectories of length $T$ from an initial distribution $\PP[X_{1}]$, then $\nu = \frac{1}{T}\sum_{i = 1}^{T} \PP[X_{i}]$.
    \item If --- as in the molecular dynamics or Lorenz 63 experiments --- one samples from an {\em invariant} distribution $\pi$ such that $\PP[X_{t}] = \pi \implies \PP[X_{t + 1}] = \pi$, one has $\nu = \mu = \pi$.
\end{itemize}
The evolution operator $\evop$ maps functions from $\Lsq{\mu}$ into $\Lsq{\nu}$, that is $\evop : \Lsq{\mu} \to \Lsq{\nu}$. We will need the following assumption:
\begin{assumption}
    The evolution operator $\evop$ is Hilbert-Schmidt. 
\end{assumption}
Let now $\enc: \X \to \R^{d}$ be an encoder whose components are square-integrable with respect to both $\mu$ and $\nu$ --- $\enc_{i} \in \Lsq{\mu}$ and $\Lsq{\nu}$ --- and define the linear operators
\eqals{
    \Phi_{\mu} : \Lsq{\mu} &\to \R^{d} \qquad
    f \mapsto f = (\scal{f}{\enc_i}_{\Lsq{\mu}})_{i = 1}^{d}. \\
    \Phi_{\nu}^{*} : \R^{d} &\to \Lsq{\nu} \qquad
    z \mapsto \sum_{i = 1}^{d}\enc_{i}(\cdot)z_{i}. \\
}
We will now show this simple fact:
\begin{lemma}
    The loss function~\cref{eq:abstract_loss} is equivalent to the following operator learning loss:
\eqals{
\varepsilon(\enc, \P) = \norb{\evop - \Phi^*_{\nu} \P \Phi_{\mu}}^{2}_{\hs}.
}
\end{lemma}
\begin{proof}
    First let's notice that by direct calculation one obtains 
    \eqals{
        \Phi_{\nu}\Phi^*_{\nu} = \EE_{\nu} \left[\enc(x)\enc(x)^{\top}\right] \qquad \Phi_{\nu}\evop\Phi^*_{\mu} = \EE_{\rho} \left[\enc(x)\enc(y)^{\top}\right],
    }
    where $\rho(dx, dy) = p(dy|x)\nu(dx)$ is the joint distribution of $(X_{t}, X_{t +1})$.
    
    By definition of Hilbert-Schmidt norm we have
    \eqals{
        \norb{\evop - \Phi^* \P \Phi}^{2}_{\hs} &= \norb{\evop}_{\hs}^{2} -2\tr\left[\evop^*\Phi^*_{\nu}\P\Phi_{\mu}\right] + \tr\left[\Phi^*_{\mu} \P^{\top}\Phi_{\nu}\Phi^*_{\nu}\P\Phi_{\mu}\right] \\
        &= \norb{\evop}_{\hs}^{2} -2\tr\left[\Phi_{\mu}\evop^*\Phi^*_{\nu}\P\right] + \tr\left[\Phi_{\mu}\Phi^*_{\mu} \P^{\top}\Phi_{\nu}\Phi^*_{\nu}\P\right] \\
        &= \norb{\evop}_{\hs}^{2} -2\EE_{(x, y) \sim \rho}\left[\tr\left[\enc(y)\enc(x)^{\top}\P\right]\right] + \EE_{(x, y) \sim \mu \otimes \nu}\left[\tr\left[\enc(y)\enc(y)^{\top} \P^{\top}\enc(x)\enc(x)^{\top}\P\right] \right]\\
        &= \norb{\evop}_{\hs}^{2} -2\EE_{(x, y) \sim \rho}\left[\scal{\enc(x)}{\P\enc(y)}\right] + \EE_{(x, y) \sim \mu \otimes \nu}\left[\scal{\enc(x)}{\P\enc(y)}^{2} \right],\\
    }
    where we repeatedly used the cyclic property of the trace.
\end{proof}
Another interesting property of~\cref{eq:abstract_loss}, is that we can compute in its minimizer with respect to $\P$ in closed form. This follows by noticing that $\varepsilon(\enc, \P)$ is convex in $\P$. Taking the gradient (see, for example~\cite{minka2000old}) one has:
\eqals{
\nabla_{\P} \varepsilon(\enc, \P) &= -2\EE_{(x, y) \sim \rho}\left[\enc(y)\enc(x)^{\top}\right] + 2\EE_{(x, y) \sim \mu \otimes \nu}\left[\enc(y)\enc(y)^{\top} \P^{\top}\enc(x)\enc(x)^{\top}\right] \\
&=  -2\cov{YX} + 2\cov{Y}\P^{\top}\cov{X} \\
& \implies \P_{*} = \cov{X}^{-1}\cov{YX}\cov{Y}^{-1} \qquad (\nabla_{\P} \varepsilon(\enc, \P_{*}) = 0).
}
\section{Experimental details}
The experiments have been performed on the following hardware:
\begin{itemize}
    \item 1 Node with 32 cores Ice Lake at 2.60 GHz, 4 $\times$ NVIDIA Ampere A100 GPUs, 64 GB and 512 GB RAM. 
    \item 1 Node with 20 cores Xeon Silver 4210 at 2.20 GHz, 4 $\times$ NVIDIA Tesla V100 GPUs, 16 GB and 384 GB RAM.
    \item A workstation equipped with a i7-5930K CPU at 3.50 GHz, 2 $\times$ NVIDIA GeForce GTX TITAN X GPUs, 12 GB and 32 GB of RAM.
\end{itemize}
\subsection{Lorenz '63}
\textbf{Training details.} We generated a single long trajectory of 15,000 time steps using the \texttt{kooplearn} 1.1.3~\cite{kooplearn} implementation of Lorenz '63 dynamical system, with default parameters. To ensure convergence to a system's attractor~\cite{tucker1999lorenz}, we discarded the first 1,000 time steps. Also, to obtain approximately time-independent segments for training, validation and testing, we further discarded 1,000 time steps between each split. In total, 10,000 time steps were used for training, and 1,000 time steps each for validation and testing. 

Our encoder consisted of an MLP with an input layer of size 3, two hidden layers with 16 units each, and an 8-dimensional latent space, using ReLU activation functions. The model was trained for 100 epochs using the AdamW optimizer, with an initial learning rate of $10^{-3}$ decayed to $10^{-4}$ via a cosine schedule, a batch size of 512, and a lag time of 10 time steps.

\textbf{Baseline methods.} We compared our approach against the following baseline models:
\begin{itemize}
    \item \textbf{Linear Least Squares (LinLS).} A linear regression model trained directly on the raw input features without any nonlinear transformation or latent representation.
    \item \textbf{Kernel Ridge Regression (KRR)~\cite{kostic2022learning}.} We trained a KRR model with a Gaussian kernel, using the bandwidth estimated via the median heuristic~\cite{garreau2017large}. The model was trained with a rank of 8, a Tikhonov regularization parameter of $10^{-6}$, and using Arnoldi iterations.
    \item \textbf{VAMPNets~\cite{mardt2018vampnets}.} Trained with the same MLP encoder architecture as our approach. The VAMP loss was defined by a Schatten norm of 2 and no centering of the covariances.
    \item \textbf{Dynamic Autoencoder (DAE)~\cite{lusch2018deep}.} Trained with the same MLP encoder architecture as in our approach; the decoder was defined symmetrically. The DAE loss weights (reconstruction, prediction, and linear evolution) were set to 1.
    \item \textbf{Consistent Autoencoder~\cite{azencot2020forecasting}.} Trained with the same MLP encoder architecture as in our approach; the decoder was defined symmetrically. The CAE loss weights (reconstruction, prediction, backward prediction, linear evolution, and consistency) were set to 1.
\end{itemize}

For all deep learning-based baselines (VAMPNets, DAE, CAE), models were trained for 100 epochs using a batch size of 512. VAMPNets used the AdamW optimizer with a learning rate of $10^{-4}$, while DAE and CAE used the Adam optimizer with a learning rate of $10^{-3}$. All baselines were implemented using \texttt{kooplearn} 1.1.3~\cite{kooplearn}.

\textbf{Additional analysis.} In \cref{fig:l63_efuns}, we show the leading eigenfunctions of the transfer operators computed using our method and the baseline approaches. These visualizations highlight qualitative differences in the learned spectral structures, offering insight into the dynamics captured by each method. The leading eigenfunction of KRR, VAMPNets, and DAE is constant and associated with the stable attractor. Our method, LinLS, and KRR, find an eigenfunction with eigenvalue $\approx .996$ which clearly separates the two lobes of the attractor. %

\begin{figure}[h!]
\centering
\includegraphics[width=\textwidth]{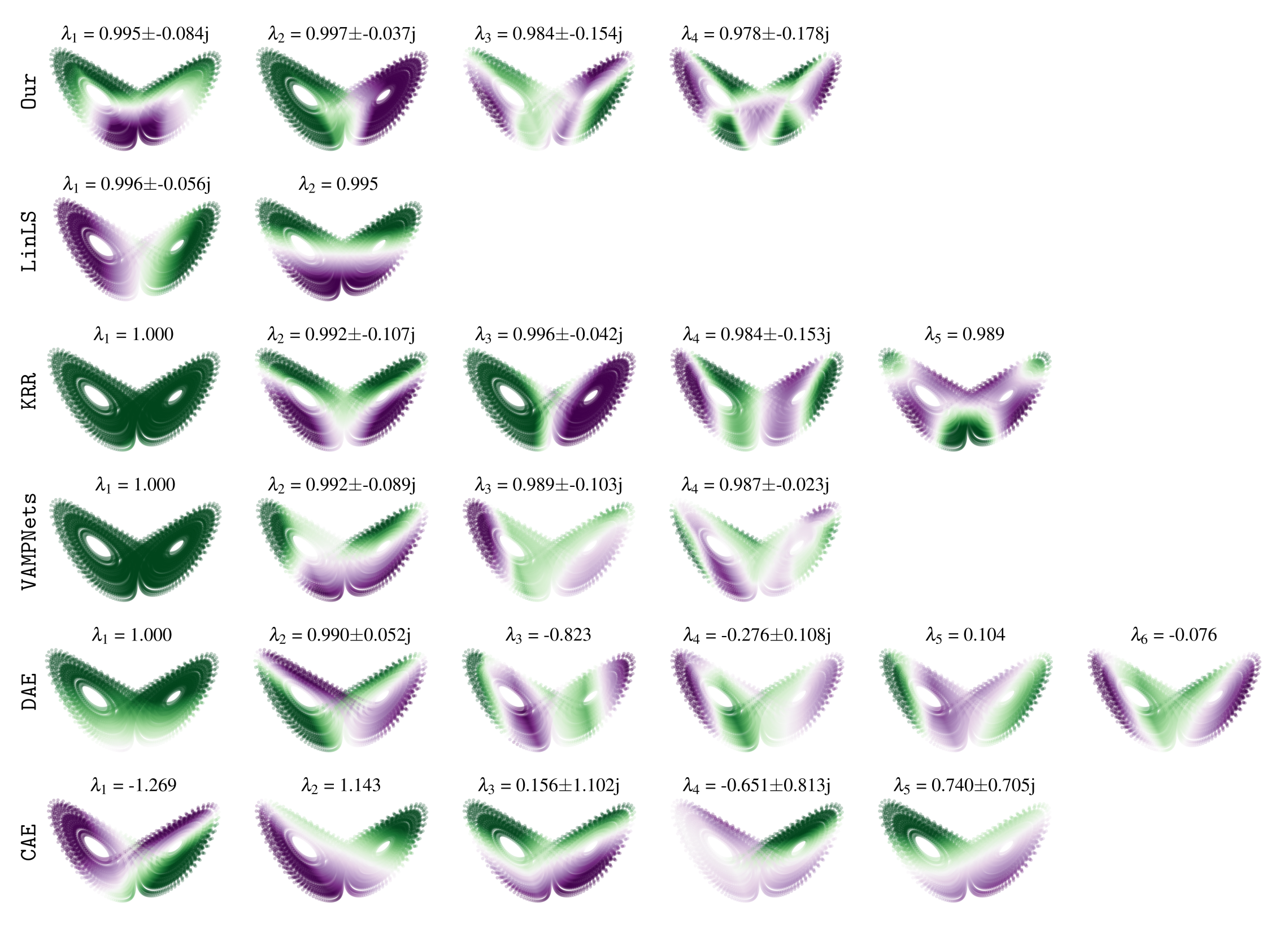}
\caption{Leading eigenfunctions computed by our and baseline approaches. Each row corresponds to a different method, and each column shows an eigenfunction ordered by decreasing eigenvalue magnitude.}
\label{fig:l63_efuns}
\end{figure}

\subsection{Protein folding}
\textbf{Training details.} We used data from~\cite{lindorff2011fast}, which can be requested directly to De Shaw Research and are available without charge for academic usage. 
Our encoder consisted of a SchNet~\cite{schutt2017schnet} graph neural network with 3 interaction blocks, 16 RBF functions and an hidden dimension of 64. The model was trained with an AdamW optimizer with starting learning rate of $10^{-2}$ decaying to $10^{-4}$ with a cosine schedule, using the \texttt{mlcolvar}~\cite{bonati2023unified} library.

\textbf{Additional analysis.} To understand to what mode is associated the leading eigenfunction $\Psi_1$, in Fig.~\ref{fig:si-trpcage} we correlated it with two physical quantities associated with the folding, which are the Root-Mean-Square-Deviation (RMSD) and the Radius of Gyration, see~\cref{fig:si-trpcage}. Furthermore, to obtain a finer understanding, we used sparse linear models to approximate the CVs via LASSO regression. This yields a surrogate model which is a linear combination of a few physical descriptors, hence interpretable. To choose the regularization strength, we computed the Mean Square Error of the surrogate model versus the number of features, see~\cref{fig:si-lasso}.

We performed LASSO regression on a set of contact functions determining the presence of hydrogen bonds. The features selected by this procedure, as well as a snapshot of the protein where these features are highlighted, are reported in~\cref{fig:si_table_and_snapshot}. Interestingly, some of the selected features pertain to side-chain interactions, a piece of information that would have been impossible to get using only C$_{\alpha}$ atoms to train the encoder.

\begin{figure}[h!]
\centering
\begin{minipage}{0.49\textwidth}
\centering
\includegraphics[width=\textwidth]{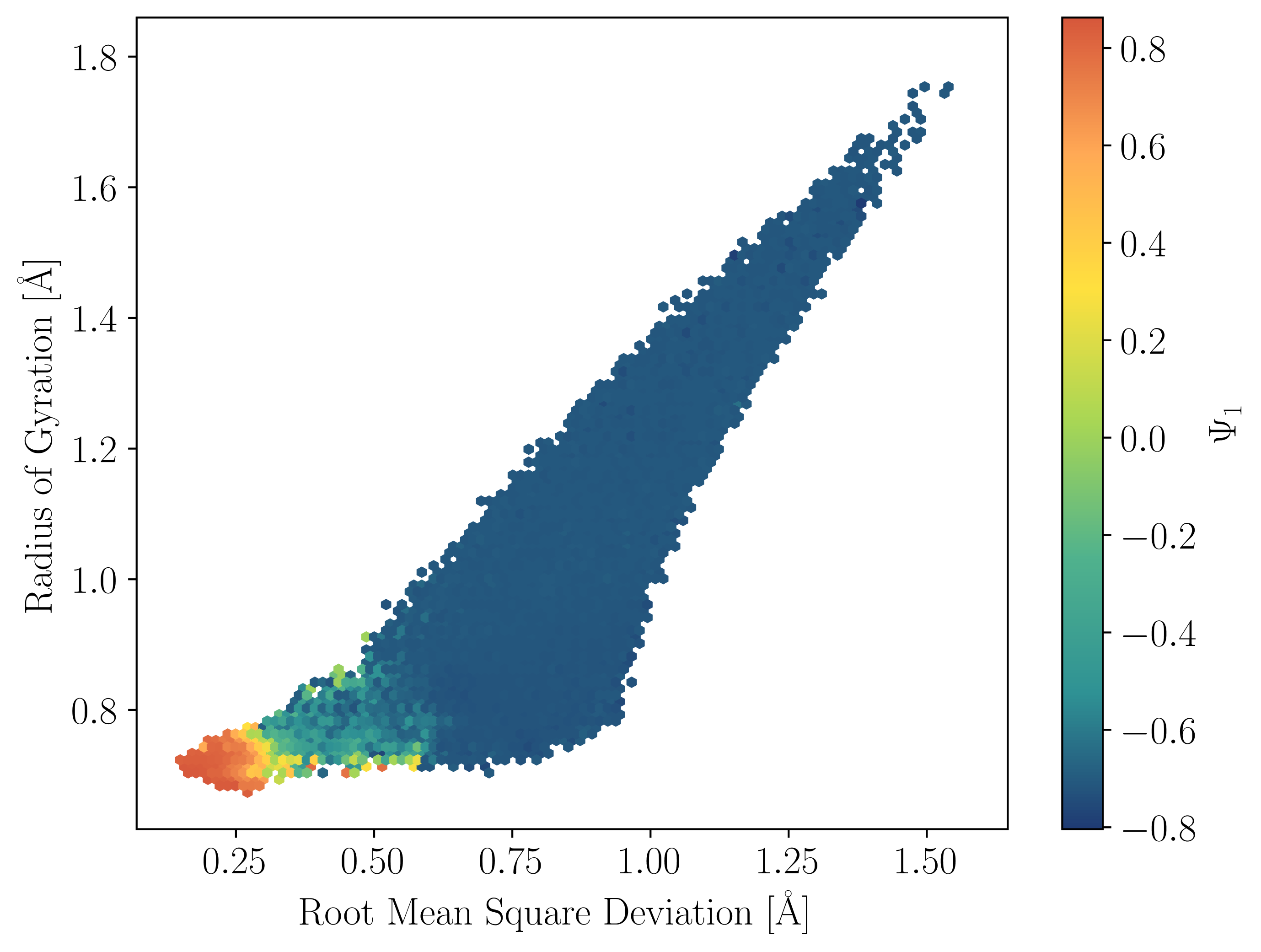}
\caption{The value of the leading eigenfunction $\Psi_1$ of the evolution operator is highly correlated with the RMSD and Radius of Gyration of the Trp-cage protein.}
\label{fig:si-trpcage}
\end{minipage}
\hfill
\begin{minipage}{0.49\textwidth}
\centering
\includegraphics[width=\textwidth]{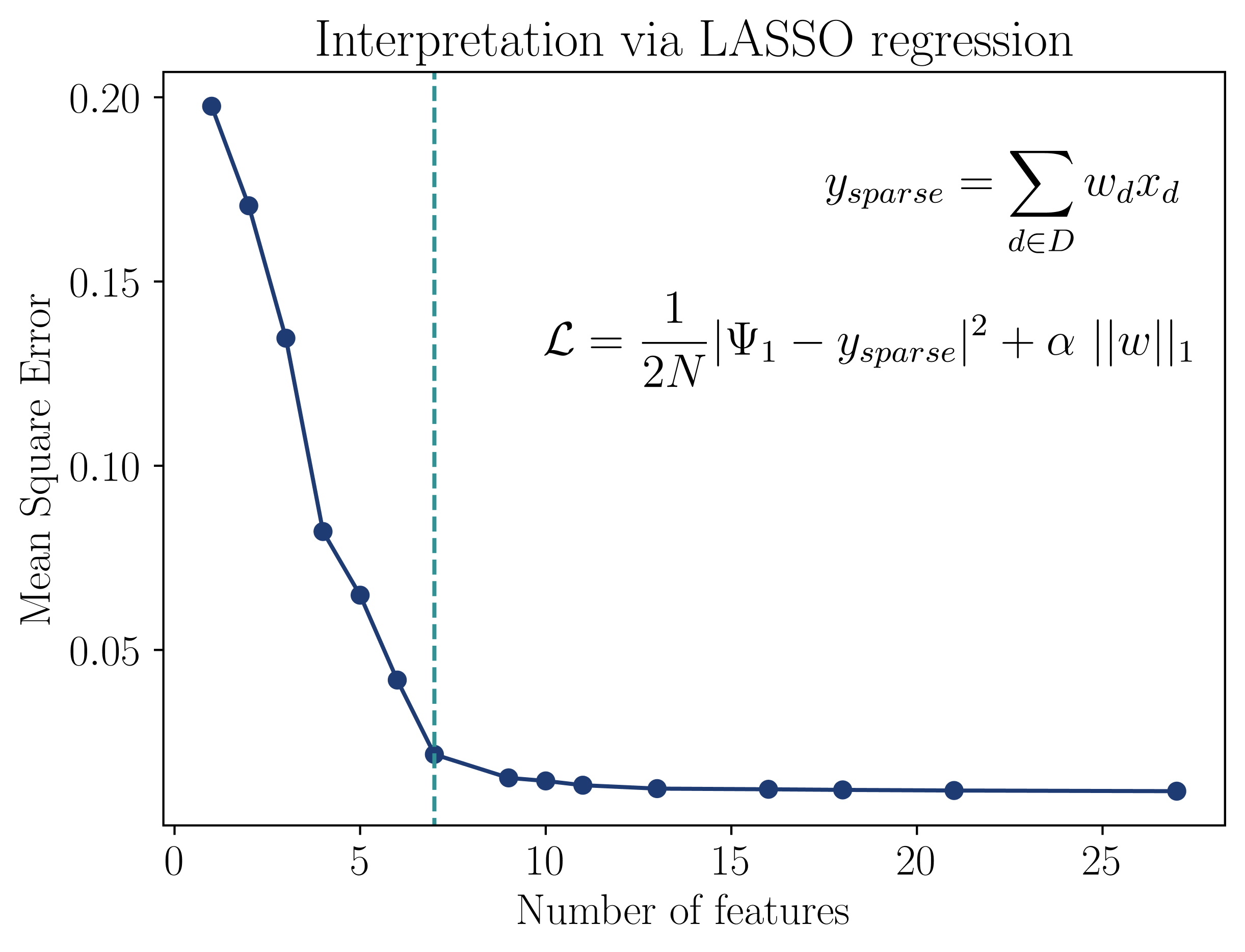}
\caption{MSE of approximating $\Psi_1$ by LASSO regression on meaningful physical descriptors. For Trp-cage we constructed a library of hydrogen-bond contact functions. The selected descriptors are reported in~\cref{fig:si_table_and_snapshot}}
\label{fig:si-lasso}
\end{minipage}
\end{figure}

\begin{figure}[h]
\centering
\setlength{\tabcolsep}{4pt}
\footnotesize
\begin{minipage}[c]{0.68\textwidth}
    \centering
    \begin{tabular}{lc}
        \toprule
        \textbf{Physical descriptors (H-bonds)} & \textbf{Normalized Coefficient} \\
        \midrule
        GLY10-O -- SER13-N & 0.307 \\
        GLY11-O -- ARG16-N & 0.294 \\
        TRP6-O -- GLY11-N & 0.170 \\
        TRP6-NE1s (sidechain) -- ARG16-O & 0.109 \\
        GLN5-O -- ASP9-N & 0.073 \\
        TRP6-NE1s (sidechain) -- PRO17-O & 0.044 \\
        TRP6-NE1s (sidechain) -- PRO18-N & 0.002 \\
        \bottomrule
    \end{tabular}
    \vspace{0.2cm}
    \label{tab:eq1_norm_coeff}
\end{minipage}
\hfill
\begin{minipage}[c]{0.28\textwidth}
    \centering
    \includegraphics[width=\linewidth]{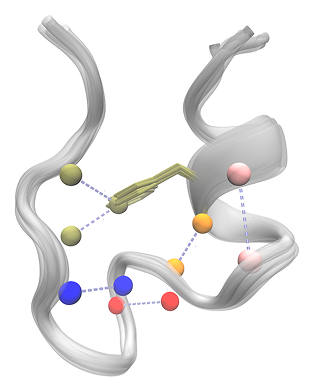}
\end{minipage}
\caption{Normalized hydrogen-bond coefficients selected by the LASSO model (left) and representative structural snapshot (right) with the features highlighted.}
\label{fig:si_table_and_snapshot}
\end{figure}

\subsection{Ligand binding}
\textbf{Simulations details.}
We selected a subset of host-guest systems for the SAMPL5 challenge \cite{bannan2016blind,yin2017overview} to evaluate our method's performance, including three ligands (G1, G2, G3) and the octa-acid calixarane host (OAMe). Simulations were run in GROMACS 2024.5 \cite{Abraham2015} patched with PLUMED 2.9.3 \cite{tribello2014plumed}. Systems were built using the GAFF \cite{wang2004development} force field with RESP \cite{bayly1993well} charges, solvated in a cubic TIP3P \cite{jorgensen1983comparison} water box \SI{40.27}{\angstrom} length, containing 2100 water molecules. System charge balanced with \ce{Na+} ions. Our timestep is \SI{2}{\femto\second} and the temperature is set to \SI{300}{\kelvin} via a velocity rescale thermostat \cite{bussi2007canonical} with a coupling time of \SI{0.1}{\pico\second}. All simulations aligned the host's vertical axis $h$ with the box axis and centered coordinates on virtual atom V1. All production simulations were initiated from the dissociated state of each ligand. Trajectories were terminated when the ligand fully rebounded into the binding pocket (defined as host-guest distance $h < \qty{6}{\angstrom}$). For each ligand, we performed 10 independent production trajectories, with coordinates saved every 500 steps.

In our simulations, we applied a funnel restraint \cite{limongelli2013funnel} to limit the conformational space explored by the ligand in the unbound state, in turn accelerating the binding process. The parameters are identical to those used in previous studies \cite{rizzi2021role}. We define $h$ as the projection of each ligand along the binding axis, treated as its radial component. For $h \geq \SI{10}{\angstrom}$, the funnel surface is a cylinder with radius $R_\text{cyl} = \SI{2}{\angstrom}$ along the vertical axis. For $h < \SI{10}{\angstrom}$, the funnel opens into a conical shape with a $45^\circ$ angle, defined by $r = 12 - h$. The force acting on a displacement $x$ from the funnel surface is harmonic:
\[
F_{\text{funnel}} = -k_F x \quad \text{with} \quad k_F = \SI{20}{\kilo\joule\per\mol\per\angstrom\squared}
\]
An additional harmonic restraint prevents the ligand from escaping too far from the host, enforcing an upper boundary:
\[
F_{\text{upper}} = -k_U (h - 18) \quad \text{for} \quad h > \SI{18}{\angstrom}, \quad \text{with} \quad k_U = \SI{40}{\kilo\joule\per\mol\per\angstrom\squared}
\]

The data will be released to ensure the reproducibility of the experiment.

\textbf{Training details.} 
Our encoder consisted of a SchNet~\cite{schutt2017schnet} graph neural network with 3 interaction blocks, 16 RBF functions, and a hidden dimension of 64 with an AdamW optimizer with starting learning rate of $10^{-2}$ decaying to $10^{-4}$ with a cosine schedule.

\begin{figure}[h!]
\centering
\includegraphics[width=0.9\textwidth]{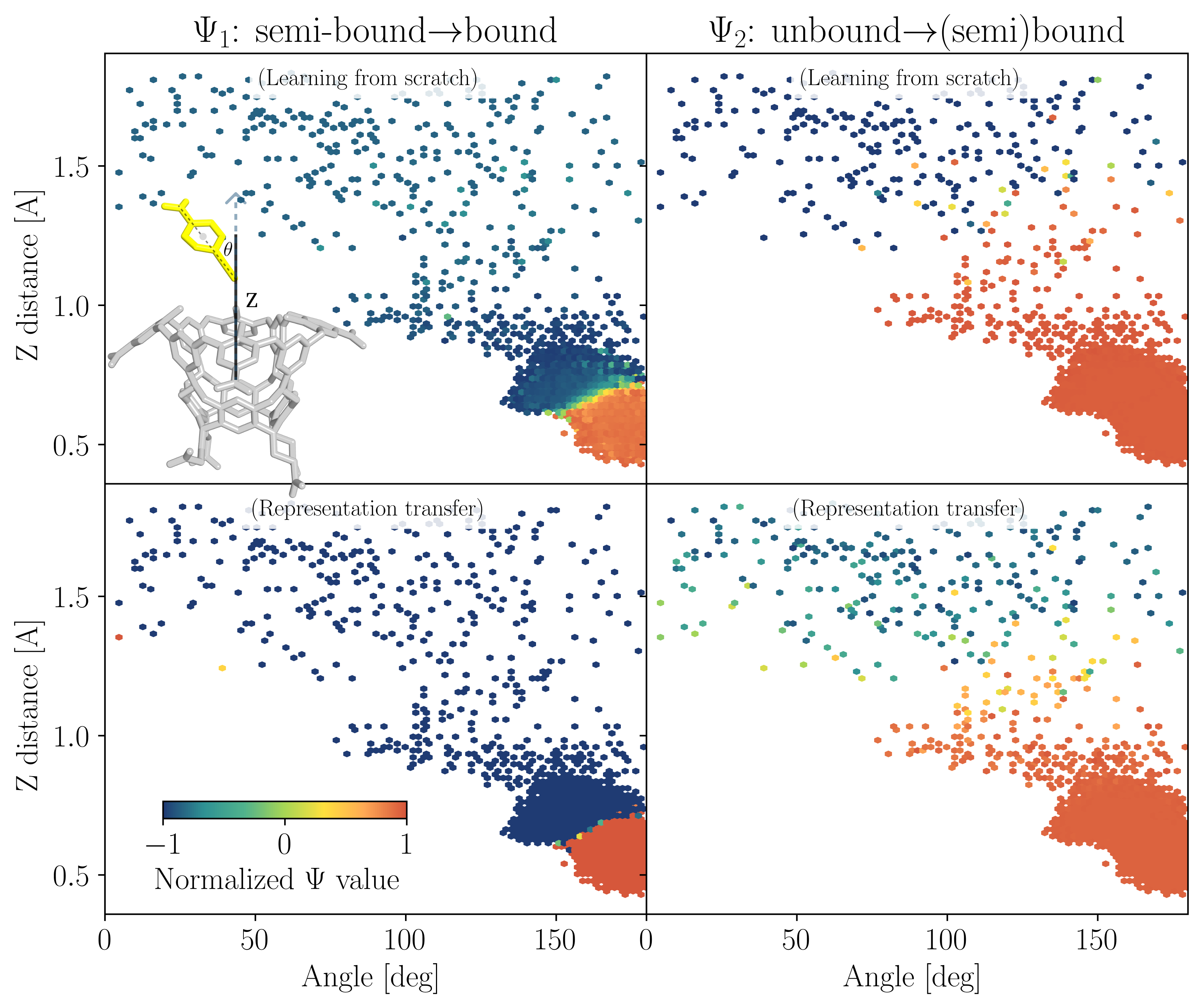}
\caption{Analysis of the leading eigenfunctions in the space of the host-guest distance $z$ and the ligand orientation $\theta$ for $\Psi_1$ (left) and $\Psi_2$ (right). The first row contains the results obtained from training from scratch the representation, while the second row contains the case in which it is transferred from other systems.}
\label{fig:si-calix}
\end{figure}

\textbf{Additional analysis.} In Fig.~\ref{fig:si-calix} we inspect the two leading eigenfunctions of the evolution operator by correlating them with two physical descriptors connected to the binding: the distance along the $z$ direction between the center of mass of the host and the guest and the angle of the ligand with respect to the $z$ axis (see figure in the inset). These results allow us to correlate the $\Psi_1$ eigenfunction to the transition between the semi-bound pose to the native one, which is due to the presence of trapped water molecules inside the pocket~\cite{rizzi2021role,bhakat2017resolving}. The second eigenfunction $\Psi_2$ is instead associated with the binding process. Furthermore, we compared the eigenfunctions obtained by training the representation from scratch on the G2 ligand with the case in which this is transferred from other ligands (G1 and G3), obtaining a remarkable agreement. The ligands G1, G2, and G3 are represented in~\cref{fig:si-ligands}

\begin{figure}[h!]
\centering
\includegraphics[width=0.9\textwidth]{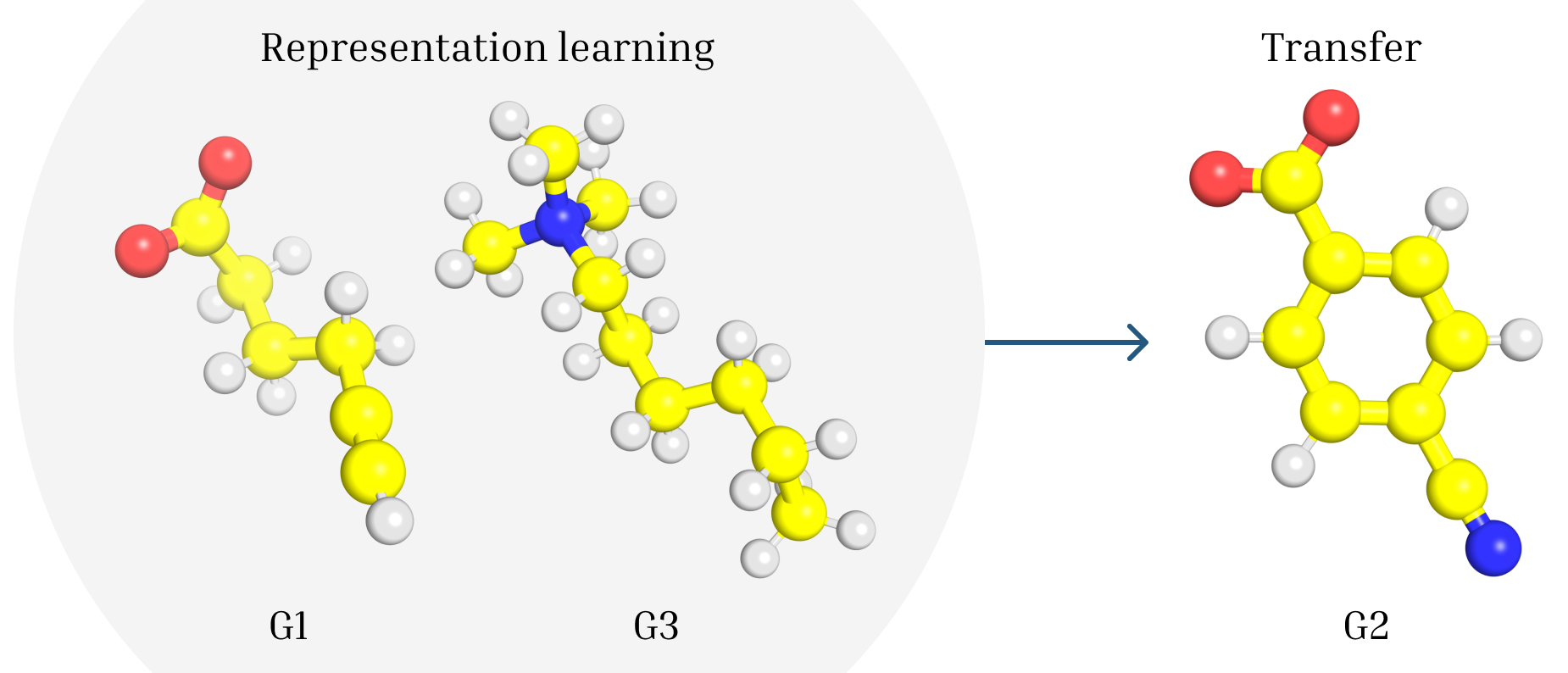}
\caption{The three different molecules studied in the ligand-binding experiment.}
\label{fig:si-ligands}
\end{figure}

\subsection{Climate modeling}
\textbf{Training details.} Following the methodology outlined in~\cite{noaaClimatePrediction}, we compute SST* from sea surface temperature (SST) data provided by the ORAS5 reanalysis~\cite{zuo2019ecmwf}, as made available through the ChaosBench dataset~\cite{nathaniel2024chaosbench}. The dataset spans a 45-year period (1979–2023) at a spatial resolution of 1.5\degree, resulting in a time series of 540 monthly snapshots, each with dimensions $121 \times 240$. Data from 1979 to 2016 was used for training, while the 2017–2023 period was reserved for validation.

We employed a ResNet-18~\cite{he2016deep} architecture as the encoder and a linear layer mapping to a 128-dimensional latent space as the predictive network. We trained the encoder adding simplicial normalization~\cite{lavoie2022simplicial} with dimension 4, spectral normalization~\cite{miyato2018spectral} applied to the predictive layer, and gradient clipping with a maximum norm of 0.2. To incorporate temporal memory, we augmented the input by appending the previous SST* snapshot as an additional input channel. The model was trained with a lag time of one month for 1,000 epochs using the AdamW optimizer, an initial learning rate of $10^{-3}$ decayed to $10^{-5}$ via a cosine schedule, and a batch size of 64. We analyzed the model achieving the highest validation score. \cref{tab:CM_evals} shows the leading eigenvalues of the learned transfer operator.

The hyperparameters reported above were selected following a grid search; \cref{tab:CM_HPO} summarizes the ranges explored during the search.

In this experiment, we encountered training instability, where runs that achieved good performance on the validation set were difficult to reproduce consistently. We attribute this instability to overfitting, and we are actively working on improving the training pipeline for this experiment.

\setlength{\tabcolsep}{4pt}
\begin{table}
\footnotesize
\begin{minipage}[t]{0.48\textwidth}
\begin{tabular}[t]{lccccc}
\toprule
\textbf{Idx} & \textbf{Re} & \textbf{Im} & \textbf{Abs} & \textbf{Decor (yr)} & \textbf{Freq (yr)} \\
\midrule
3 & 0.86 & 0.50 & 0.99 & 12.21 & 1.00 \\
4 & 0.86 & -0.50 & 0.99 & 12.21 & -1.00 \\
5 & 0.49 & 0.85 & 0.98 & 4.64 & 0.50 \\
6 & 0.49 & -0.85 & 0.98 & 4.64 & -0.50 \\
7 & 0.00 & 0.97 & 0.97 & 2.60 & 0.33 \\
8 & 0.00 & -0.97 & 0.97 & 2.60 & -0.33 \\
10 & -0.47 & -0.83 & 0.96 & 1.81 & -0.25 \\
9 & -0.47 & 0.83 & 0.96 & 1.81 & 0.25 \\
0 & -0.82 & 0.47 & 0.95 & 1.49 & 0.20 \\
1 & -0.82 & -0.47 & 0.95 & 1.49 & -0.20 \\
2 & -0.94 & 0.00 & 0.94 & 1.43 & 0.17 \\
11 & 0.93 & 0.00 & 0.93 & 1.18 & 0.00 \\
12 & 0.80 & 0.46 & 0.93 & 1.09 & 1.00 \\
13 & 0.80 & -0.46 & 0.93 & 1.09 & -1.00 \\
14 & 0.45 & 0.79 & 0.91 & 0.88 & 0.50 \\
15 & 0.45 & -0.79 & 0.91 & 0.88 & -0.50 \\
16 & 0.00 & 0.88 & 0.88 & 0.68 & 0.33 \\
17 & 0.00 & -0.88 & 0.88 & 0.68 & -0.33 \\
\bottomrule
\end{tabular}
\end{minipage}%
\hfill
\begin{minipage}[t]{0.48\textwidth}
\begin{tabular}[t]{lccccc}
\toprule
\textbf{Idx} & \textbf{Re} & \textbf{Im} & \textbf{Abs} & \textbf{Decor (yr)} & \textbf{Freq (yr)} \\
\midrule
18 & -0.42 & 0.72 & 0.84 & 0.47 & 0.25 \\
19 & -0.42 & -0.72 & 0.84 & 0.47 & -0.25 \\
20 & -0.67 & 0.38 & 0.77 & 0.32 & 0.20 \\
21 & -0.67 & -0.38 & 0.77 & 0.32 & -0.20 \\
22 & -0.75 & 0.00 & 0.75 & 0.29 & 0.17 \\
23 & 0.71 & 0.00 & 0.71 & 0.24 & 0.00 \\
24 & 0.54 & 0.33 & 0.63 & 0.18 & 0.95 \\
25 & 0.54 & -0.33 & 0.63 & 0.18 & -0.95 \\
26 & 0.29 & 0.55 & 0.62 & 0.17 & 0.48 \\
27 & 0.29 & -0.55 & 0.62 & 0.17 & -0.48 \\
28 & -0.02 & 0.56 & 0.56 & 0.15 & 0.33 \\
29 & -0.02 & -0.56 & 0.56 & 0.15 & -0.33 \\
30 & -0.28 & 0.46 & 0.54 & 0.14 & 0.25 \\
31 & -0.28 & -0.46 & 0.54 & 0.14 & -0.25 \\
32 & -0.44 & 0.26 & 0.51 & 0.12 & 0.20 \\
33 & -0.44 & -0.26 & 0.51 & 0.12 & -0.20 \\
34 & -0.48 & 0.00 & 0.48 & 0.11 & 0.17 \\
\bottomrule
\end{tabular}
\end{minipage}

\caption{Leading eigenvalues of the transfer operator learned on SST\textsuperscript{*} data. Each eigenvalue is expressed in terms of its real (Re), imaginary (Im), and absolute (Abs) components. The associated decorrelation times and oscillation frequencies (in years) are also reported. Eigenvalues are listed in descending order with respect to their absolute value, and those with a decorrelation time shorter than 1/12 years—the sampling frequency—were discarded.}
\label{tab:CM_evals}
\end{table}
\setlength{\tabcolsep}{4pt}
\begin{table}
\centering
\footnotesize
\begin{tabular}{lc}
\toprule
\textbf{Hyperparameter} & \textbf{Search Range} \\
\midrule
Latent dimensions & \texttt{[32, 64, 128, ..., 1024]} \\
Max gradient clipping norm & \texttt{[None, 0.1, 0.2, 0.5]} \\
Normalization of linear layer & \texttt{[False, True]} \\
Regularization & \texttt{[0, 1e-5, ..., 1e-2]} \\
Simplicial normalization dimensions & \texttt{[0, 2, ..., 16]} \\
History length & \texttt{[0, 1, 2]} \\
\bottomrule
\end{tabular}
\vspace{0.2cm}
\caption{Hyperparameter ranges explored during grid search for the climate modeling task.}
\label{tab:CM_HPO}
\end{table}

\end{document}